\documentclass[lettersize,journal]{IEEEtran}
\usepackage{amsmath,amsfonts}
\usepackage{algorithmic}
\usepackage{algorithm}
\usepackage{array}
\usepackage{textcomp}
\usepackage{stfloats}
\usepackage{url}
\usepackage{verbatim}
\usepackage{graphicx}
\usepackage{cite}
\usepackage{amssymb}
\usepackage{amsthm}
\usepackage{bm}
\usepackage{multirow}
\usepackage{subfigure}
\usepackage{hyperref}
\usepackage{cleveref}
\usepackage{gensymb}
\usepackage{booktabs}
\usepackage{soul}
\usepackage{color}
\newcommand{\pengxin}[1]{{\color{black}{#1}}}

\newtheorem{remark}{Remark}

\newtheorem{theorem}{Theorem}

\begin{document}

\title{Online Test-Time Adaptation of Spatial-Temporal Traffic Flow Forecasting}

\author{Pengxin Guo, Pengrong Jin, Ziyue Li, Lei Bai, and Yu Zhang,~\IEEEmembership{Member, IEEE}
\thanks{Pengxin Guo is with the Department of Statistics and Actuarial Science, The University of Hong Kong, Hong Kong 999077, China (e-mail: guopx@connect.hku.hk).}
\thanks{Pengrong Jin is with the Department of Mathematics, Southern University of Science and Technology, Shenzhen 518055, China (e-mail: jinpr@mail.sustech.edu.cn).}
\thanks{Ziyue Li is with the Department of Information Systems, University of Cologne, Cologne 50923, Germany (e-mail: zlibn@wiso.uni-koeln.de).}
\thanks{Lei Bai is with the Shanghai Artificial Intelligence Laboratory, Shanghai 200232, China (e-mail: baisanshi@gmail.com).}
\thanks{Yu Zhang is with the Department of Computer Science and Engineering, Southern University of Science and Technology, Shenzhen 518055, China, and also with the Peng Cheng Laboratory, Shenzhen 518000, China (e-mail:  yu.zhang.ust@gmail.com).}
\thanks{This work was done during the first author's internship at Shanghai Artificial Intelligence Laboratory.}
\thanks{Corresponding authors: Lei Bai; Yu Zhang.}}




\maketitle

\begin{abstract}
\pengxin{Accurate spatial-temporal traffic flow forecasting is crucial in aiding traffic managers in implementing control measures and assisting drivers in selecting optimal travel routes.}
Traditional \pengxin{deep-learning based} methods for \pengxin{traffic flow} forecasting typically rely on historical data to train their models, which are then used to make predictions on future data. However, the performance of the trained model usually degrades due to the temporal drift between the historical and future data. To make the model trained on historical data better adapt to future data in a fully online manner, this paper conducts the first study of the online test-time adaptation techniques for spatial-temporal \pengxin{traffic flow} forecasting problems. To this end, we propose an \textbf{A}daptive \textbf{D}ouble \textbf{C}orrection by \textbf{S}eries \textbf{D}ecomposition (ADCSD) method, which first decomposes the output of the trained model into seasonal and trend-cyclical parts and then corrects them by two separate modules during the testing phase using the latest observed data \textit{entry by entry}. In the proposed ADCSD method, instead of fine-tuning the whole trained model during the testing phase, a lite network is attached after the trained model, and only the lite network is fine-tuned in the testing process each time a data entry is observed. Moreover, to satisfy that different time series variables may have different levels of temporal drift, two adaptive vectors are adopted to provide different weights for different time series variables. Extensive experiments on four real-world \pengxin{traffic flow} forecasting datasets demonstrate the effectiveness of the proposed ADCSD method. The code is available at \href{https://github.com/Pengxin-Guo/ADCSD}{https://github.com/Pengxin-Guo/ADCSD}.

\end{abstract}


\section{Introduction}
\label{sec:intro}


\begin{figure}[t]
\centering
\includegraphics[width=\linewidth]{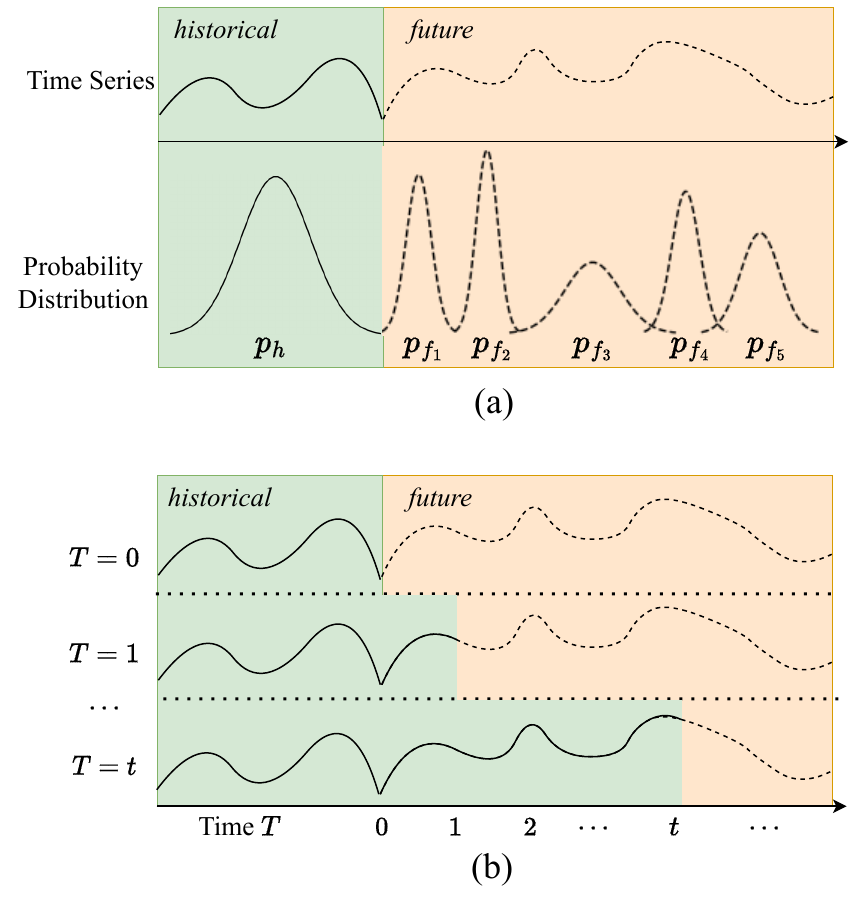}
\caption{(a). The temporal drift problem in non-stationary time series. The raw time series data is multivariate in reality. As our focus is on the distribution of the future data rather than the historical data, we assume that the historical data follows one distribution $p_h$, although it is possible for it to follow multiple distributions. For the future data, the distribution will changes over time, i.e., $p_{f_1} \neq p_{f_2} \neq p_{f_3} \neq p_{f_3} \neq p_{f_4} \neq p_{f_5} \neq p_h. $ (b). Data identity change. With the time goes by, the future data will become historical data.}
\vskip -0.1in
\label{fig:intro_all}
\end{figure}

\IEEEPARstart{S}{patial-temporal} \pengxin {traffic flow forecasting is a key component of intelligent transportation systems (ITS) and has received significant attentions in recent years \cite{10323235, lv2014traffic, zhao2019t, tedjopurnomo2020survey, jiang2022graph,nagy2018survey}.}
Most \pengxin{deep learning-based traffic flow forecasting} methods train models on historical data and make predictions on future data without modifying the trained model \cite{guo2019attention,bai2020adaptive,woo2022cost,pdformer}. However, the performance of those models is usually unsatisfactory due to the non-stationarity of time series data, i.e., temporal drift between the historical and future data \cite{kuznetsov2014generalization,du2021adarnn} (see Figure \ref{fig:intro_all}(a)). This temporal drift could stem from changes in upstream processes, sensors, materials, natural drift, or variable relation shifts, etc. To solve the temporal drift problem, some works \cite{du2021adarnn,duan2023combating} aim to learn distribution-aware knowledge to aid in model training  by assuming that future data follow a different distribution from the historical data. However, it is challenging to assume the availability of the distribution for future data, as it often changes over time, as illustrated in Figure \ref{fig:intro_all}(a). Moreover, those methods do not take into consideration the relativity of data identity. That is, the definition of history and future in time series data is relative, as the historical data accrues over time, as shown in Figure \ref{fig:intro_all}(b).

\IEEEpubidadjcol
To handle those problems faced in spatial-temporal \pengxin{traffic flow} forecasting, we study the Online Test-Time Adaptation (OTTA) techniques \cite{liang2023comprehensive,sun2020test,niu2022efficient}. \pengxin{OTTA techniques are widely adopted in computer vision (CV) community to address the distribution shift issue by continuously updating and refining the model during the testing phase, based on the feedback and new data encountered during deployment \cite{wang2021tent,gandelsman2022test,gong2022note,wang2022continual}. This adaptive process allows the model to learn and adjust its predictions in real-time, improving its performance and generalization to new data.} Though OTTA has been studied in the CV community, to the best of our knowledge, there is no work for spatial-temporal \pengxin{traffic flow} forecasting and in this paper, we will give the first try on it. 
There are some OTTA methods \cite{wang2021tent,gandelsman2022test,gong2022note,wang2022continual} proposed to solve CV problems, but directly applying them to solve spatial-temporal \pengxin{traffic flow} forecasting problem may not give satisfactory performance as they do not consider the complex spatial and temporal correlations present in spatial-temporal data and the distribution shift problem that occurs over time. Furthermore, as mentioned above, the \textit{relativity of data identity} property (see Figure \ref{fig:intro_all}(b)) in time series data is also a crucial distinction from CV problems, which is defined as follows: at time $t+1$, we can observe the true label of data $\mathbf{x}_t$. Thus, 
after we make predictions on data $\mathbf{x}_t$ at time $t$, we can utilize the ground-truth label information $\mathbf{y}_t$ at time $t+1$ to update the model, which is not taken into consideration in existing OTTA methods.


To apply OTTA techniques to the spatial-temporal \pengxin{traffic flow} forecasting problem, there are several challenges. Firstly, simply fine-tuning the trained model on spatial-temporal data streams may impair the performance of the model, as we only have access to one data point at a time. Secondly, spatial-temporal data often suffer from the distribution shift problem \cite{du2021adarnn,duan2023combating}, meaning that the data distribution tends to change over time (see Figure \ref{fig:intro_all}(a)). Thirdly, spatial-temporal data is typically composed of multiple time series variables with different levels of temporal drift \pengxin{and the relationship between different time series variables is complicated} \cite{bai2020adaptive,tian2021spatial}, making the problem more complex. Therefore, applying OTTA to spatial-temporal data is a non-trivial problem and requires further studies.

To address above problems, we propose the Adaptive Double Correction by Series Decomposition (ADCSD) method, which first decomposes the output of the trained model into seasonal and trend-cyclical parts and then corrects them by two separate modules. Specifically, to solve the first issue discussed in the previous paragraph, inspired by Fast Weight Layers (FWLs) \cite{clark2022meta}, instead of fine-tuning the whole trained model, we attach a lite network after the trained model and fine-tune the lite network. To tackle the second issue, we first take the idea of decomposition \cite{cleveland1990stl}, a standard time series analysis method, to decompose the series into seasonal and trend-cyclical parts and then adopt two modules to correct them, respectively. Benefiting from that such decomposition can ravel out the entangled temporal patterns and highlight the inherent properties of time series \cite{hyndman2018forecasting}, we can easily correct the temporal drift between the future data and historical data. Then we add the original output of the trained model with the corrected seasonal and trend-cyclical parts to obtain the final corrected output to conquer the temporal drift problem. For the third issue mentioned above, we design two adaptive vectors to \pengxin{implicitly involve the interaction of spatial information between time series variables and} balance the weight between the original output and corrected seasonal and trend-cyclical parts to fit different time series variables with different levels of temporal drift. In summary, our contributions are as follows.
\begin{itemize}
    \item To the best of our knowledge, we are the first to apply OTTA to spatial-temporal \pengxin{traffic flow} forecasting problems. 
    \item To deal with the temporal drift problem, we propose the ADCSD method, which can be used by various spatial-temporal \pengxin{traffic flow} deep models in a plug-and-play manner to improve their performance.
    \item Extensive experiments on both graph-based and grid-based traffic flow forecasting datasets demonstrate the effectiveness of the proposed ADCSD method.
\end{itemize}

\section{Related Work}
\label{sec:rel_work}

\subsection{Spatial-Temporal Traffic Flow Forecasting} Spatial-temporal \pengxin{traffic flow} forecasting, which combines both spatial and temporal information to predict how 
\pengxin{the traffic flow} 
will change over time and space, has attracted much attention in recent years \cite{li2020long, amato2020novel, liu2022msdr, nagy2018survey, jones2017machine, longo2017crowd}. To perform reliable and accurate \pengxin{traffic flow} forecasting, many models have been proposed \cite{guo2019attention, bai2020adaptive, guo2021learning, pdformer}.
For example, Attention based Spatial-Temporal Graph Convolutional Networks (ASTGCN) \cite{guo2019attention} that consists of three independent components to respectively model three temporal properties of traffic flows, i.e., recent, daily-periodic and weekly-periodic dependencies, Adaptive Graph Convolutional Recurrent Network (AGCRN) \cite{bai2020adaptive} that can capture fine-grained spatial and temporal correlations in traffic series automatically based on two module (i.e., Node Adaptive Parameter Learning module and Data Adaptive Graph Generation module) and recurrent networks, Attention based Spatial-Temporal Graph Neural Network (ASTGNN) \cite{guo2021learning} that consists of a novel self-attention mechanism to capture the temporal dynamics of traffic data and a dynamic graph convolution module to capture the spatial correlations in a dynamic manner, Propagation Delay-aware Dynamic Long-range Transformer (PDFormer) \cite{pdformer} that consists of a spatial self-attention module to capture the dynamic spatial dependencies and two graph masking matrices to highlight spatial dependencies from short- and long-range views, etc. 
However, all of those methods train models on the historical data and make predictions on the future data without modifying the trained model. Thus, the trained model usually cannot give satisfactory performance due to the non-stationary of time series data, i.e., temporal drift between the historical and future data \cite{kuznetsov2014generalization,du2021adarnn}. Some works attempt to solve this problem \cite{du2021adarnn,you2021learning,arik2022self,kim2022reversible,duan2023combating,wang2023koopman,bai2023temporal}. For example,  Du et al. \cite{du2021adarnn} propose Adaptive RNNs (AdaRNN) that first introduces a temporal distribution characterization (TDC) algorithm to split the training data into several diverse periods and then adopts a temporal distribution matching (TDM) algorithm to dynamically reduce  the distribution divergence. Kim et al. \cite{kim2022reversible} propose Reversible Instance Normalization (RevIN) which is a generally applicable normalization-and-denormalization method with learnable affine transformation and can remove and restore the statistical information of a time-series instance. Duan et al. \cite{duan2023combating} propose Hyper TimeSeries Forecasting (HTSF) that exploits the hyper layers to learn the best characterization of the distribution shifts and generate model parameters for the main layers to make accurate predictions. Wang et al. \cite{wang2023koopman} propose Koopman Neural Forecaster (KNF) that based on Koopman theory for time-series data with temporal distributional shifts can capture the global behaviors and evolve over time to adapt to local changing distributions. However, the assumption that the future data follows one distribution is not realistic. Besides, those models are deployed with fixed learned parameters at test time, which cannot adapt to the changing data distributions. In this work, we draw inspiration from the setting of online test-time adaptation \cite{liang2023comprehensive} and fine-tune the trained model during the testing phase to adapt to the future data distribution, which is in line with the characteristics of spatial-temporal data. To the best of our knowledge, we are the first to study the OTTA for spatial-temporal \pengxin{traffic flow} forecasting to solve the temporal drift problem.

\subsection{Online Test-Time Adaptation} Online Test-Time Adaptation (OTTA), which only can access the trained model and unlabelled test data, is to improve the performance of the model during inference \cite{hu2021mixnorm,azimi2022self,hong2023mecta}. 
\pengxin{Typically, during the training of machine learning models, we use offline datasets for training and adjust the model based on the training data. However, when the model is deployed in real-world scenarios, it may encounter new and unseen data distributions or face changing environments. This can lead to a degradation in model performance. OTTA addresses this issue by continuously updating and refining the model during the testing phase, based on the feedback and new data encountered during deployment. This adaptive process allows the model to learn and adjust its predictions in a real-time manner to improve its performance and generalization to new data.}
During past years, many (online) test-time adaptation methods have been proposed, such as Fully Test-Time Adaptation by Entropy Minimization (TENT) \cite{wang2021tent} that optimizes the model to improve the confidence measured by the entropy of its predictions, Test-Time Training with Masked Autoencoders (TTT-MAE) \cite{gandelsman2022test} that use masked autoencoders for the one-sample learning problem, NOn-i.i.d. TEst-time adaptation scheme (NOTE) \cite{gong2022note} that adopts instance-aware batch normalization for out-of-distribution samples and prediction-balanced reservoir sampling to simulate an i.i.d. data stream from a non-i.i.d. stream in a class-balanced manner, Laplacian Adjusted Maximum-likelihood Estimation (LAME) \cite{boudiaf2022parameter} that aims at providing a correction to the output probabilities of a classifier, etc. However, all of those methods focus on computer vision problems and do not deal with spatial-temporal \pengxin{traffic flow} problems. Different from computer vision data, there are spatial and temporal relationships in spatial-temporal data and we should take into consider when we study the OTTA setting for spatial-temporal data, which is what the proposed ADCSD method does.

\section{Methodology}
\label{sec:method}
In this section, we first formalize the OTTA setting for spatial-temporal \pengxin{traffic flow} forecasting, and then present the proposed ADCSD method to learn under the OTTA setting.

\subsection{Problem Settings
}

We consider the OTTA of spatial-temporal \pengxin{traffic flow} forecasting problem. Specifically, we have access to a history model $f$ trained on the historical data and future test data $\mathcal{X} = \{\mathbf{x}_t, \mathbf{y}_t\}_{t=1}^{n}$, where $n$ is the total number of future test data, $\mathbf{x}_t \in \mathbb{R}^{N \times T \times C}$ is the \pengxin{traffic flow} data of $N$ \pengxin{traffic nodes} at $T$ time slices with $C$ features, and $\mathbf{y}_t \in \mathbb{R}^{N \times T' \times C}$ is the corresponding label.\footnote{$T' = 1$ implies a one-step prediction task, while $T' > 1$ implies a multi-step prediction task.} For the OTTA of spatial-temporal \pengxin{traffic flow} forecasting problem, we can first access data $\mathbf{x}_t$ at time $t$ and need to make prediction for  $\mathbf{x}_t$. Then at time $t+1$, we have access to the true label $\mathbf{y}_t$ of data $\mathbf{x}_t$ since we can observe it after time $t$ and can utilize this labeled data to update the model. In other words, at each time, we first need to make a prediction for this data and then can utilize the labeled data to update the model. 
In the following, we will describe our method that utilizes the labeled data to update the model.

\subsection{Overview}

Since there usually are temporal drift between time series data, it is necessary to correct the output of a model trained on the historical data to perform well on the future data. To achieve this, we propose the ADCSD method. Specifically, inspired by Fast Weight Layers (FWLs) \cite{clark2022meta}, instead of fine-tuning the whole trained model, we attach a lite network behind the trained model and only fine-tune the lite network. In addition, based on series decomposition \cite{cleveland1990stl}, the proposed ADCSD method can separate the series data into trend-cyclical and seasonal parts to deeply analyze the pattern of series data and correct them by two modules. Finally, two adaptive vectors are introduced to \pengxin{implicitly model spatial interactions between traffic nodes and} balance the weight between the original output and corrected seasonal and trend-cyclical parts to fit that different \pengxin{traffic nodes} have different levels of temporal drift. The overall framework is shown in Figure \ref{fig:framework}. In the following sections, we introduce all the components one by one.

\begin{figure*}[!htbp]
\centering
\includegraphics[width=.8\textwidth]{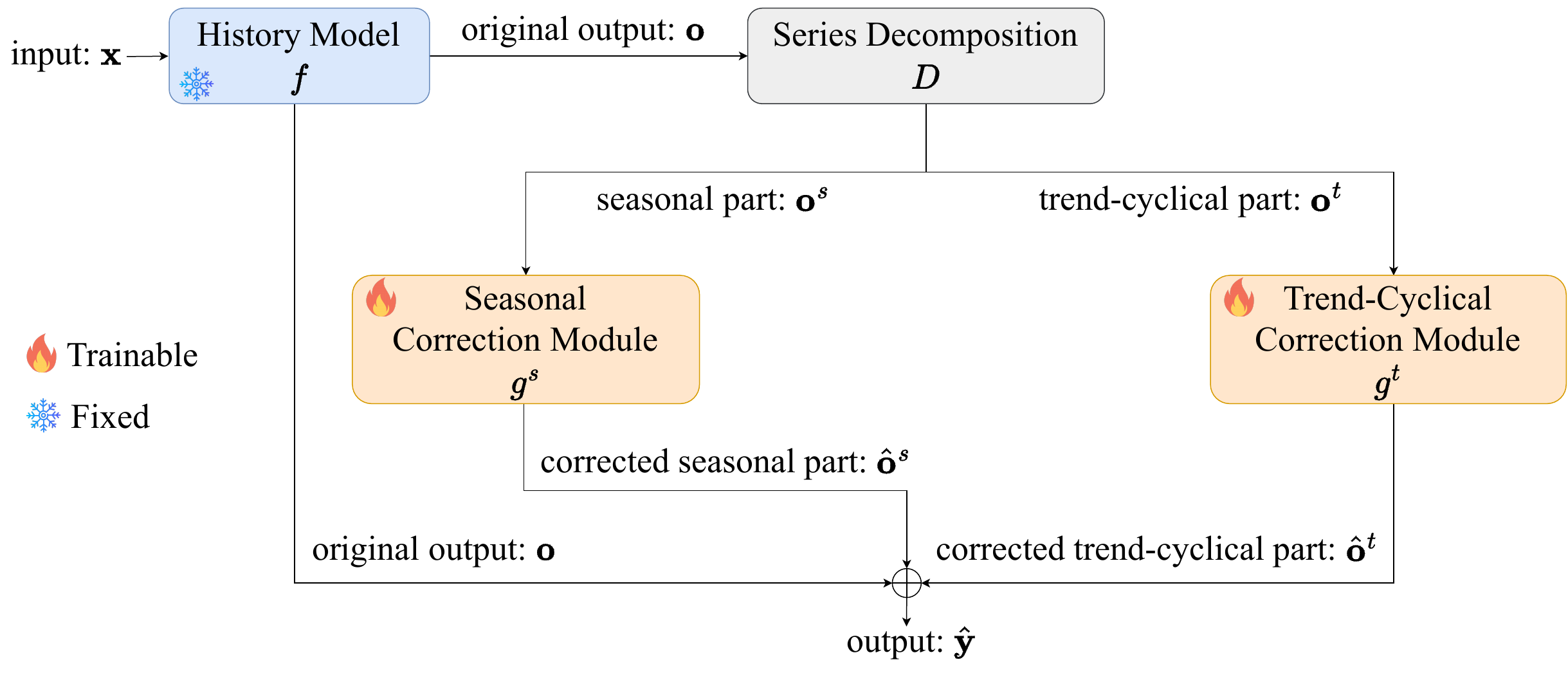}
\caption{The framework of the proposed ADCSD method.}
\label{fig:framework}
\end{figure*}

\subsection{Series Decomposition}

For the OTTA of spatial-temporal \pengxin{traffic flow} forecasting, we have access to a history model $f$ trained on the historical data and future data $\mathcal{X}$ need to be predicted. To correct the output of the history model, we take the idea of series decomposition \cite{cleveland1990stl}, which can separate time series data into trend-cyclical and seasonal parts to learn complex temporal patterns in long-term forecasting. Specifically, for future data $\mathbf{x}$,\footnote{For notation simplicity, we omit the subscript time $t$.} we first obtain the output $\mathbf{o}$ of the history model $f$ as
\begin{equation} \label{eq:original_output}
    \mathbf{o} = f(\mathbf{x}).
\end{equation}
Then, to correct the output, we first decompose the output into seasonal and trend-cyclical parts by series decomposition. In detail, we first obtain the trend-cyclical part $\mathbf{o}^t$ of the original output via moving average as
\begin{equation*}
    \mathbf{o}^t = \text{AvgPool}(\text{Padding}(\mathbf{o})),
\end{equation*}
where $\text{AvgPool}(\cdot)$ denotes the average pooling operation on the time dimension and $\text{Padding}(\cdot)$ denotes the padding operation that is used to keep the length unchanged. Then the seasonal part $\mathbf{o}^s$ is equal to the difference between the original output and the trend-cyclical part, i.e.,
\begin{equation*}
    \mathbf{o}^s = \mathbf{o} - \mathbf{o}^t.
\end{equation*}
In short, we use the following equation to denote the series decomposition as
\begin{equation} \label{eq:decompose}
    \mathbf{o}^s, \mathbf{o}^t = D(\mathbf{o}).
\end{equation}
After the series decomposition, we can break down the task of correcting the original output into correcting its seasonal and trend-cyclical parts, respectively, which facilitates the overall correction process.

\subsection{Correction Modules}

To overcome the temporal drift problem, two correction modules are adopted to correct the seasonal and trend-cyclical parts, respectively. Specifically, for the seasonal part, the correction module is defined as 
\begin{equation} \label{eq:corr_sea}
    \hat{\mathbf{o}}^s = g^s(\mathbf{o}^s),
\end{equation}
where $\hat{\mathbf{o}}^s$ is the corrected seasonal part and $g^s$ denotes the seasonal correction module. The seasonal correction module used in our work is a neural network of two fully-connected layers with layer normalization \cite{ba2016layer} and Gaussian Error Linear Unit (GELU) \cite{hendrycks2016gaussian} activation function in the middle. Similarly, for the trend-cyclical part, the correction module is defined as 
\begin{equation} \label{eq:corr_trend}
    \hat{\mathbf{o}}^t = g^t(\mathbf{o}^t),
\end{equation}
where $\hat{\mathbf{o}}^t$ is the corrected trend-cyclical part and $g^t$ denotes the trend-cyclical correction module. The trend-cyclical correction module has the same architecture as the seasonal correction module. With the integration of these two correction modules, the proposed ADCSD method enables the adaptation of the output to effectively align with the distribution of future data. By dynamically adjusting the model's predictions based on the characteristics of the incoming data, we can enhance its ability to accurately capture the evolving patterns and trends present in the spatial-temporal domain. This adaptability ensures that our method remains robust and performs well when the data distribution changes over time.

\subsection{Adaptive Combination}

\pengxin{Previous works \cite{zhao2019t, bai2020adaptive, pdformer} have shown that the relationship between different nodes is dynamic but not static. However, the mechanisms they employed to model the interaction of spatial information between nodes, such as graph neural networks or spatial attentions, are relatively heavyweight due to matrix multiplications involved.}
Moreover, since different \pengxin{nodes} may have different levels of temporal drift and a shared module has difficulty in learning the different levels of temporal drift among different time series. 
\pengxin{To address these challenges, we adopt a lightweight approach that includes the introduction of two adaptive vectors. These vectors serve to implicitly incorporate the interaction of spatial information between traffic nodes and assign varying weights to different nodes in the corrected seasonal and trend-cyclical components, respectively, when combining them with the original output.}
Formally, let $\bm{\lambda}^s \in \mathbb{R}^N$ denote the seasonal adaptive vector and $\bm{\lambda}^t \in \mathbb{R}^N$ denote the trend-cyclical adaptive vector. The final output $\hat{\mathbf{y}}$ of our method is given as:
\begin{equation} \label{eq:predict}
    \hat{\mathbf{y}} = \mathbf{o} + \bm{\lambda}^s \hat{\mathbf{o}}^s + \bm{\lambda}^t \hat{\mathbf{o}}^t.
\end{equation}
By utilizing those two trainable vectors, we are able to learn different weights for different \pengxin{nodes} for the corrected seasonal and trend-cyclical parts, respectively, when combining them with the original output. Finally, the objective function of the proposed ADCSD method is formulated as:
\begin{equation} \label{eq:obj}
    \min_{g^s, g^t, \bm{\lambda}^s, \bm{\lambda}^t} \ell(\mathbf{y}, \hat{\mathbf{y}}),
\end{equation}
where $\ell$ denotes the loss function such as the square loss. Note that parameters of the history model $f$ are fixed and we only update the parameters in correction modules $g^s$ and $g^t$ and adaptive vectors $\bm{\lambda}^s$ and $\bm{\lambda}^t$. The entire algorithm of the proposed ADCSD method is given in Algorithm \ref{alg:method}.

\begin{algorithm}[!ht]
\renewcommand{\algorithmicrequire}{\textbf{Input:}}
\renewcommand{\algorithmicensure}{\textbf{Output:}}
\caption{Adaptive Double Correction by Series Decomposition}
\label{alg:method}
    \begin{algorithmic}[1] 
        \REQUIRE  history model $f$, future data $\mathcal{X}$; 
        
        
        \STATE Randomly initialize $g^s$ and $g^t$; 
        \STATE Initialize $\bm{\lambda}^s$ and $\bm{\lambda}^t$ as $\mathbf{0}$;
                
        \FOR {$t=1,\cdots,n$}
            \STATE Compute the original output using Eq. (\ref{eq:original_output});
            \STATE Decompose the original output into seasonal and trend-cyclical parts using Eq. (\ref{eq:decompose});
            \STATE Correct the seasonal and trend-cyclical parts respectively using Eqs. (\ref{eq:corr_sea}) and (\ref{eq:corr_trend});
            \STATE Compute the final output $\hat{\mathbf{y}}_t$ using Eq. (\ref{eq:predict});
            \STATE Update the parameters of $g^s$, $g^t$, $\bm{\lambda}^s$ and $\bm{\lambda}^t$ by minimizing problem (\ref{eq:obj});
        \ENDFOR
    \end{algorithmic}
\end{algorithm}

\subsection{Analysis}
\label{sec:theo_anal}

For parameter efficiency, the proposed ADCSD method attaches a lite network including $g^s$ and $g^t$ behind the trained model and fine-tune the lite network. Here we provide some analyses to give insights into this method.

Let $\mathcal{X} = \{\mathbf{x}_t, \mathbf{y}_t\}_{t=1}^{n}$ denotes the future \pengxin{traffic flow} data, where $\mathbf{x}_t\in \mathbb{R}^{N\times T\times C}$ and $\mathbf{y}_t\in \mathbb{R}^{N\times T'\times C}$. $f$ is the model trained on the historical data and the parameters are fixed during the testing phase. Let us consider two models. The output of model 1 is formulated as:
\begin{equation} \label{theo:model_1}
    \hat{\mathbf{y}}=f(\mathbf{x}),
\end{equation}
and that of model 2 is 
\begin{equation} \label{theo:model_2}
    \Tilde{\mathbf{y}}=f(\mathbf{x})+g(f(\mathbf{x})),
\end{equation}
where $g$ is the attached lite network. Hence, model 2 is just the proposed ADCSD method. We choose the square loss as the loss function for those two models and the loss function is defined as
\begin{equation}
    \ell(\mathbf{y},\hat{\mathbf{y}})=\frac{1}{NT'C}\sum_{i=1}^{NT'C}(y_i-\hat{y}_i)^2,
\end{equation}
where $y_i$ is the $i$-th component of $\mathbf{y}$ and $\hat{y}_i$ is the $i$-th component of $\hat{\mathbf{y}}$. For training losses of those two models, we have the following result.

\begin{theorem} \label{theo:theorem_1}
We can find suitable function $g$ satisfying that the training loss of model 2 is lower than model 1, that is, we have
\begin{equation}
    \ell(\mathbf{y},\Tilde{\mathbf{y}}) < \ell(\mathbf{y},\hat{\mathbf{y}}).
\end{equation}
\end{theorem}

\begin{proof}
Let $\ell(\mathbf{y},\hat{\mathbf{y}})=\frac{1}{NT'C}\sum_{i=1}^{NT'C}(y_i-\hat{y}_i)^2=\frac{1}{NT'C}\sum_{i=1}^{NT'C}(y_i-f_i(\mathbf{x}))^2$, where $y_i$ is the $i$-th component of $\mathbf{y}$, $\hat{y}_i$ is the $i$-th component of $\hat{\mathbf{y}}$, and $f_i$ is the $i$-th component of function $f$. Similarly, $\ell(\mathbf{y},\Tilde{\mathbf{y}})=\frac{1}{NT'C}\sum_{i=1}^{NT'C}(y_i-\Tilde{y}_i)^2=\frac{1}{NT'C}\sum_{i=1}^{NT'C}(y_i-f_i(\mathbf{x})-g_i(f(\mathbf{x})))^2$, where $\Tilde{y}_i$ is the $i$-th component of $\Tilde{\mathbf{y}}$, and $g_i$ is the $i$-th component of function $g$. 
Then we have  
\begin{equation}
\begin{split}
	&NT'C[\ell(\mathbf{y},\hat{\mathbf{y}})-\ell(\mathbf{y},\Tilde{\mathbf{y}})]\nonumber\\
    =&\sum[(y_i-f_i(\mathbf{x}))^2-(y_i-f_i(\mathbf{x})-g_i(f(\mathbf{x})))^2] \\
    =&\sum[y_i^2-2y_if_i(\mathbf{x})+f^2_i(\mathbf{x})-(y_i^2+f^2_i(\mathbf{x})+g_i^2(f(\mathbf{x})) \\
    &- 2y_if_i(\mathbf{x})-2y_ig_i(f(\mathbf{x}))+2f_i(\mathbf{x})g_i(f(\mathbf{x})))]\\
    =&\sum[2y_ig_i(f(\mathbf{x}))-g_i^2(f(\mathbf{x}))-2f_i(\mathbf{x})g_i(f(\mathbf{x}))]\\
    =&\sum g_i(f(\mathbf{x}))[2y_i-f_i(\mathbf{x})-g_i(f(\mathbf{x}))].
\end{split}
\end{equation}
Let $g_i(f(\mathbf{x}))=G_i$ and $2y_i-f_i(\mathbf{x})=C$, where $C$ is a constant since $g_i$ and $f_i$ are fixed for any given $i$. Then above equation becomes a function of $G_i$, i.e.,
\begin{equation*}
    G_i(C-G_i).
\end{equation*}
It is a conic and has two roots $G_i=0$ and $G_i=C$. By choosing suitable $G_i\in(0,C)$ or $(C,0)$, we can always have $\sum G_i(C-G_i)> 0$, which implies $\ell(\mathbf{y},\hat{\mathbf{y}})-\ell(\mathbf{y},\Tilde{\mathbf{y}})> 0$.
\end{proof}

\begin{remark}
Theorem \ref{theo:theorem_1} implies that for a fixed model, attaching a lite trainable network will incur a lower training loss than that without the lite network. From the perspective of the model capacity, model 1 is a reduced version of model 2 by setting the parameters of the attached network to be zero and hence model 2 has a larger capacity than model 1, making model 2 possess a large chance to have a lower training loss. This conclusion is further verified in our experiments (i.e., Section \ref{sec:ablation_study}).
\end{remark}

To provide a theoretical analysis of the necessity of adding the original output to the final output. Let us consider the third model without the original output. Thus, the output of model 3 is formulated:
\begin{equation} \label{theo:model_3}
    \Bar{\mathbf{y}}=g(f(\mathbf{x})).
\end{equation}
This model adopts the same loss function as models 1 and 2. For training losses of models 2 and 3, we have the following result.

\begin{theorem} \label{theo:theorem_2}
    We can find suitable function $g$ satisfying that the training loss of model 2 is lower than model 3, that is 
\begin{equation}
    \ell(\mathbf{y},\Tilde{\mathbf{y}}) < \ell(\mathbf{y},\Bar{\mathbf{y}}).
\end{equation}
\end{theorem}

\begin{proof}
Let $\ell(\mathbf{y},\Bar{\mathbf{y}})=\frac{1}{NT'C}\sum_{i=1}^{NT'C}(y_i-\Bar{y}_i)^2=\frac{1}{NT'C}\sum_{i=1}^{NT'C}(y_i-g_i(f(\mathbf{x})))^2$, where $y_i$ is the $i$-th component of $\mathbf{y}$, $\Bar{y}_i$ is the $i$-th component of $\Bar{\mathbf{y}}$, and $g_i$ is the $i$-th component of the function $g$. Similarly,  $\ell(\mathbf{y},\Tilde{\mathbf{y}})=\frac{1}{NT'C}\sum_{i=1}^{NT'C}(y_i-\Tilde{y}_i)^2=\frac{1}{NT'C}\sum_{i=1}^{NT'C}(y_i-f_i(\mathbf{x})-g_i(f(\mathbf{x})))^2$, where $\Bar{y}_i$ is the $i$-th component of $\Bar{\mathbf{y}}$, and $f_i$ is the $i$-th component of the function $f$.
Then, in order to prove
\begin{equation*}
    \ell(\mathbf{y},\Tilde{\mathbf{y}}) < \ell(\mathbf{y},\Bar{\mathbf{y}}).
\end{equation*}
We only need to prove
\begin{equation*}
    \ell(\mathbf{y},\Bar{\mathbf{y}}) - \ell(\mathbf{y},\Tilde{\mathbf{y}}) > 0.
\end{equation*}
Then, we have
\begin{equation} \small
\begin{split}{}
    & NT'C [\ell(\mathbf{y},\Bar{\mathbf{y}})-\ell(\mathbf{y},\Tilde{\mathbf{y}})]\nonumber\\
    =&\sum[(y_i-g_i(f(\mathbf{x})))^2-(y_i-f_i(\mathbf{x})-g_i(f(\mathbf{x})))^2] \\
    =&\sum[y_i^2-2y_ig_i(f(\mathbf{x}))+g_i^2(f(\mathbf{x}))-(y_i^2+f^2_i(\mathbf{x})+g_i^2(f(\mathbf{x})) \\
    &- 2y_if_i(\mathbf{x})-2y_ig_i(f(\mathbf{x}))+2f_i(\mathbf{x})g_i(f(\mathbf{x})))]\\
    =&\sum[2y_if_i(\mathbf{x})-f_i^2(\mathbf{x})-2f_i(\mathbf{x})g_i(f(\mathbf{x}))]\\
    =&\sum f_i(\mathbf{x})[-f_i(\mathbf{x})+2y_i-2g_i(f(\mathbf{x}))].
\end{split}
\end{equation}
Since $y_i$ and $f_i(\cdot)$ are fixed for any given $i$, by choosing $g$, we can make $f_i(\mathbf{x})$ and $[-f_i(\mathbf{x})+2y_i-2g_i(f(\mathbf{x}))$ have the same symbol, i.e.,
\begin{equation*}
    f_i(\mathbf{x})[-f_i(\mathbf{x})+2y_i-2g_i(f(\mathbf{x}))]>0.
\end{equation*}
Thus, we have $\sum f_i(\mathbf{x})[-f_i(\mathbf{x})+2y_i-2g_i(f(\mathbf{x}))]>0$, which demonstrates $\ell(\mathbf{y},\Bar{\mathbf{y}})-\ell(\mathbf{y},\Tilde{\mathbf{y}})>0$.
\end{proof}

\begin{remark}
Theorem \ref{theo:theorem_2} proves the importance and necessity of adding the original output, which is also demonstrated in our experiments (i.e., Section \ref{sec:ablation_study}). 
\end{remark}

\section{Experiments}
\label{sec:exp}


\subsection{Datasets} 

\begin{table*}[h]
\centering
\caption{Statistics of Datasets.}
\resizebox{\linewidth}{!}{
\begin{tabular}{ccccccc}
\toprule 
Datasets & \#\pengxin{Nodes} & \#Time Interval & \#Input Window & \#Output Window & Train:Val:Test & Time Range \\
\midrule
PeMS07 & 883 & 5min & 12 & 12 & 6:2:2 & 05/01/2017-08/31/2017 \\
BayArea & 699 & 5min & 12 & 12 & 5:2:3 & 01/01/2019-12/30/2019 \\
\cmidrule{1-7}
NYCTaxi & 75 (15x5) & 30min & 6 & 6 & 7:1:2 & 01/01/2014-12/31/2014 \\
T-Drive & 1024 (32x32) & 60min & 6 & 6  & 7:1:2 & 02/01/2015-06/30/2015 \\

\bottomrule
\end{tabular}
}
\label{tab:data}
\end{table*}

We evaluate the proposed ADCSD method on four real-world public spatial-temporal traffic datasets, including two graph-based highway traffic datasets, i.e., PeMS07 \cite{song2020spatial}, BayArea, and two grid-based citywide traffic datasets, i.e., NYCTaxi \cite{liu2020dynamic}, T-Drive \cite{pan2019urban}. 

\textbf{PeMS07} is collected from the Performance Measurement System (PeMS) of California Transportation Agencies (CalTrans)   \cite{chen2001freeway}. It is collected from 883 sensors installed in the California with observations of 4 months of data ranging from May 1 to August 31, 2017.

\textbf{BayArea} is also collected from PeMS of CalTrans. It is collected from 4096 sensors installed in the Bay Area with observations of 12 months of data ranging from January 1 to December 30, 2019. Then we discard the data of nodes with a missing rate greater than 0.1\%, leaving only 699 nodes.

\textbf{NYCTaxi} is a dataset of taxi trips made in New York City. It is provided by the New York City Taxi and Limousine Commission (TLC)\footnote{https://www.nyc.gov/site/tlc/about/tlc-trip-record-data.page}. The data used in our paper spans from January 1 to December 31, 2014. 

\textbf{T-Drive}\footnote{https://www.microsoft.com/en-us/research/publication/t-drive-trajectory-data-sample/} is a dataset based on taxi GPS trajectories on Beijing, jointly released by Microsoft Asia Research Institute and Peking University.  The data used in our paper spans from February 1 to June 30, 2015.

The graph-based datasets contain only the traffic flow data, and the grid-based datasets contain inflow and outflow data. 
More details about these datasets used in our paper are given in Table \ref{tab:data}. Note that the training and validation sets are regarded as historical data to train the model, and the test sets are regarded as future data.

\subsection{Experimental Settings}

\subsubsection{Baselines}
To validate the correctness and effectiveness of our method, we plug-and-play our method after several commonly used \pengxin{traffic flow} traffic forecasting models, including ASTGCN \cite{guo2019attention}, AGCRN \cite{bai2020adaptive} and PDFormer \cite{pdformer}. To the best of our knowledge, there is no comparable method for OTTA of spatial-temporal \pengxin{traffic flow} forecasting. Thus, we reimplement several OTTA methods used in computer vision area for spatial-temporal \pengxin{traffic flow} forecasting, including TENT \cite{wang2021tent} and TTT-MAE \cite{gandelsman2022test}.

\subsubsection{Implementation details}
All the baseline models (i.e, ASTGCN, AGCRN and PDFormer) are implemented based on the LibCity library \cite{wang2021libcity}. The  seasonal and trend-cyclical correction modules have the same architecture, i.e., two fully connected layers with Layer Normalization \cite{ba2016layer} and Gaussian Error Linear Unit (GELU) \cite{hendrycks2016gaussian} activation function in the middle. For training the history models, we follow the optimization details in their original paper. The optimal history model is determined based on the performance in the validation set. For testing phase, we optimize the correction modules with the Adam \cite{kingma2014adam} optimizer with an initial learning rate of $10^{-4}$. Both the epoch and batch size are set to \textbf{one} to follow the online learning setting. All the experiments are conducted on Tesla A100 GPUs.

\subsubsection{Evaluation Metrics}
We use three metrics in the experiments: (1) Mean Absolute Error (MAE), (2) Mean Absolute Percentage Error (MAPE), and (3) Root Mean Squared Error (RMSE). Missing values are excluded when calculating these metrics. When we test the models on the grid-based datasets, we filter the samples with flow values below 10, consistent with \cite{pdformer}.

\subsection{Main Results}

\begin{table} [h]
\caption{Performance on the \textbf{graph}-based datasets. A lower MSE, MAPE, or MAE indicates better performance, and the best results are highlighted in bold.}
\resizebox{\linewidth}{!}{
\begin{tabular}{c|l|ccc|ccc}
\toprule  & \multirow{2}{*}{Fine-tune method}  & \multicolumn{3}{|c}{PeMS07} & \multicolumn{3}{|c}{BayArea}  \\
\cmidrule(r){3-8}  
& & MAE & MAPE(\%) & RMSE & MAE & MAPE(\%) & RMSE   \\
\midrule
\multirow{4}{*}{\rotatebox{90}{ASTGCN}} & Test & 23.708 & 10.170 & 36.903 & 19.201 & 10.531 & 31.678 \\
& TENT \cite{wang2021tent} & 33.196 & 18.407 & 50.396 & 21.857 & 13.057 & 33.861 \\
& TTT-MAE \cite{gandelsman2022test}  & 32.306 & 15.635 & 46.456 & 20.757 & 12.154 & 32.697 \\
& ADCSD & \textbf{22.910} & \textbf{9.975} & \textbf{35.740} & \textbf{17.380} &  \textbf{9.536} & \textbf{29.111} \\
\cmidrule(r){1-8}  

\multirow{4}{*}{\rotatebox{90}{AGCRN}} & Test & 20.700 & 8.980 & 34.338 & 17.169 & 9.609 & 30.632 \\
& TENT \cite{wang2021tent} & 54.461 & 36.299 & 81.075 & 53.596 & 39.067 & 76.288 \\
& TTT-MAE \cite{gandelsman2022test} & 31.565 & 15.838 & 45.101 & 22.484 & 13.731 & 35.292 \\
& ADCSD & \textbf{20.211} & \textbf{8.623} & \textbf{33.389} & \textbf{15.507} & \textbf{8.542} & \textbf{27.706} \\
\cmidrule(r){1-8}  

\multirow{4}{*}{\rotatebox{90}{PDFormer}} & Test & 19.802 & 8.565 & 32.820 & 14.883 & 7.673 & 27.866 \\
& TENT \cite{wang2021tent} & 23.020 & 11.483 & 35.369 & 20.823 & 14.907 & 31.790 \\
& TTT-MAE \cite{gandelsman2022test} & 20.051 & 8.706 & 33.012 & 15.409 & 8.430 & 27.978 \\
& ADCSD & \textbf{19.627} & \textbf{8.393} & \textbf{32.607} & \textbf{14.667} & \textbf{7.617} & \textbf{27.531} \\
\bottomrule
\end{tabular}
}
\label{tab:main_result_1}
\end{table}

\subsubsection{Graph-based Datasets}
The comparison results with baselines of different models on the graph-based datasets are shown in Table \ref{tab:main_result_1}. Based on this table, we can make the following observations. (1) {\textbf{\textit{OTTA from CV cannot be reused directly:}} The performance of OTTA methods proposed in the computer vision area (i.e., TENT and TTT-MAE) are inferior to the original trained model (i.e., the Test method), which means these methods do not work in the spatial-temporal \pengxin{traffic flow} forecasting problem. Thus, it is necessary to design new algorithm according to the characteristics of \pengxin{traffic flow} data when applying OTTA to spatial-temporal \pengxin{traffic flow} forecasting. (2) \textbf{\textit{Fine-tuning not always works:}} The performance of TENT drops dramatically compared with the original trained model, which indicates fine-tuning all the parameters of the trained model is not a good choice for the online spatial-temporal \pengxin{traffic flow} forecasting problem. 
(3) \textbf{\textit{Effectiveness of ADCSD:}} Our proposed method ADCSD can further improve the performance of the trained models and achieve the best results on all datasets, which demonstrates the effectiveness and stability of our method. \pengxin{The incorporation of learnable components enables the proposed ADCSD method to gradually adapt and learn towards the future data distribution, thereby alleviating the issue of temporal drift.}

\begin{table*}[ht]
\centering
\caption{Performance on the \textbf{grid}-based datasets. A lower MSE, MAPE, or MAE indicates better performance, and the best results are highlighted in bold.}
\resizebox{\linewidth}{!}{
\begin{tabular}{c|l|cccccc|cccccc}
\toprule 
 &  & \multicolumn{6}{|c}{NYCTaxi} & \multicolumn{6}{|c}{T-Drive}  \\
\cmidrule(r){3-14}  
& Fine-tune Method & \multicolumn{3}{|c}{inflow} & \multicolumn{3}{|c}{outflow} & \multicolumn{3}{|c}{inflow} & \multicolumn{3}{|c}{outflow}  \\
\cmidrule(r){3-14}  
& & MAE & MAPE(\%) & RMSE & MAE & MAPE(\%) & RMSE & MAE & MAPE(\%) & RMSE & MAE & MAPE(\%) & RMSE  \\
\midrule
\multirow{4}{*}{\rotatebox{90}{ASTGCN}} & Test & 25.302 & 23.405 & 43.228 & 21.720 & 22.515 & 37.758 & 41.578 & 32.410 & 76.801 & 41.376 & 32.385 & 76.365  \\
& TENT \cite{wang2021tent} & 26.915 & 27.290 & 43.592 & 24.436 & 29.276 & 39.502 & 54.139 & 45.236 & 95.186 & 54.154 & 45.689 & 95.185  \\
& TTT-MAE \cite{gandelsman2022test} & 25.469 & 23.840 & 43.414 & 22.242 & 24.024 & 38.061 & 43.278 & 35.354 & 78.464 & 43.053 & 35.277 & 78.072  \\
& ADCSD & \textbf{25.117} & \textbf{23.280} & \textbf{42.941} & \textbf{21.480} & \textbf{22.487} & \textbf{37.315} & \textbf{41.499} & \textbf{32.328} & \textbf{76.675} & \textbf{41.248} & \textbf{32.299} & \textbf{76.166}  \\
\cmidrule(r){1-14} 

\multirow{4}{*}{\rotatebox{90}{AGCRN}} & Test &  18.299 & 16.445 & 31.948 & 15.754 & 16.211 & 27.383 & 21.374 &  16.478 & 41.066 & 21.395 & 16.459 & 41.081 \\
& TENT \cite{wang2021tent} & 41.357 & 42.575 & 67.217 & 38.109 & 45.765 & 64.413 & 73.054 & 60.799 & 126.561 & 72.923 & 60.353 & 126.684  \\
& TTT-MAE \cite{gandelsman2022test} & 20.675 & 19.105 & 35.302 & 18.052 & 18.929 & 30.755 & 28.572 & 22.574 & 52.056 & 28.588 & 22.523 & 52.121 
 \\
& ADCSD & \textbf{18.114} & \textbf{16.232} & \textbf{31.673} & \textbf{15.569} & \textbf{16.001} & \textbf{27.102} & \textbf{21.242} & \textbf{16.389} & \textbf{40.856} & \textbf{21.271} & \textbf{16.366} & \textbf{40.897} \\
\cmidrule(r){1-14} 

\multirow{4}{*}{\rotatebox{90}{PDFormer}} & Test & 17.150 & 15.320 & 30.185 & 14.778 & 15.055 & 25.903 & 21.481 & 17.504 & 38.889 & 21.515 & 17.507 & 38.914  \\
& TENT \cite{wang2021tent} & 24.256 & 24.304 & 40.333 & 21.265 & 24.858 & 35.357 & 45.906 & 35.253 & 80.966 & 45.734 & 35.034 & 80.766 \\
& TTT-MAE \cite{gandelsman2022test} & 17.221 & 15.437 & {30.163} & 14.902 & 15.204 & 25.998 & 22.164 & 18.187 & 39.515 & 22.194 &  18.193 & 39.547 \\
& ADCSD & \textbf{16.987} & \textbf{15.181} & \textbf{29.928} & \textbf{14.595} & \textbf{14.912} & \textbf{25.460} & \textbf{21.308} & \textbf{17.283} & \textbf{38.664} & \textbf{21.337} & \textbf{17.274} &  \textbf{38.694}\\

\bottomrule
\end{tabular}
}
\label{tab:main_result_2}
\end{table*}

\begin{figure*} [t]
  \centering
    \subfigure[MAE]{\includegraphics[width=0.32\textwidth]{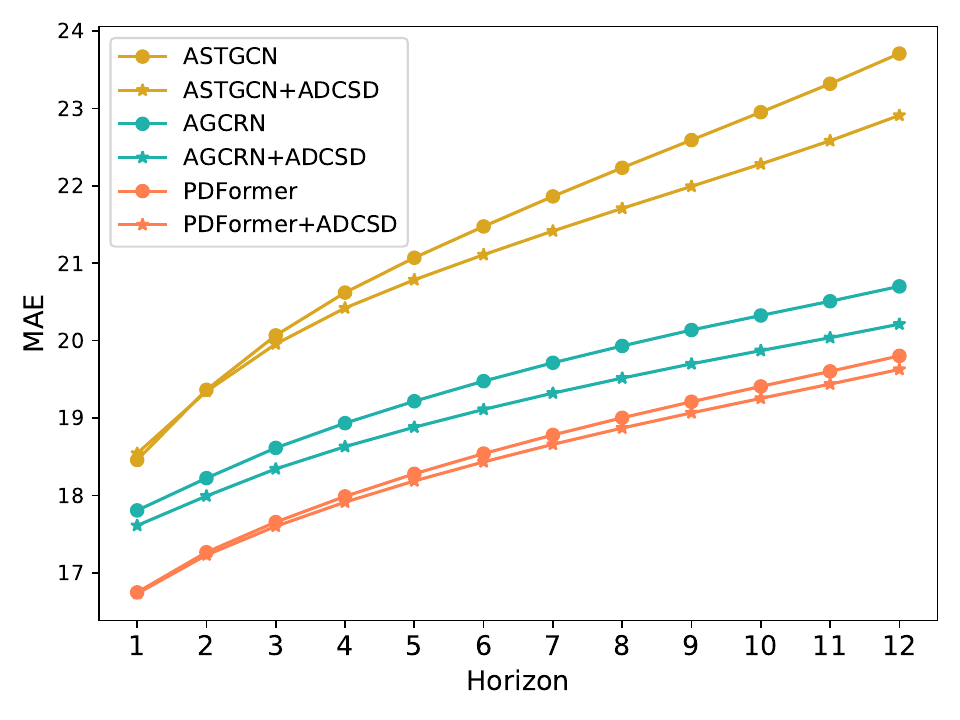}}
    \subfigure[MAPE]{\includegraphics[width=0.32\textwidth]{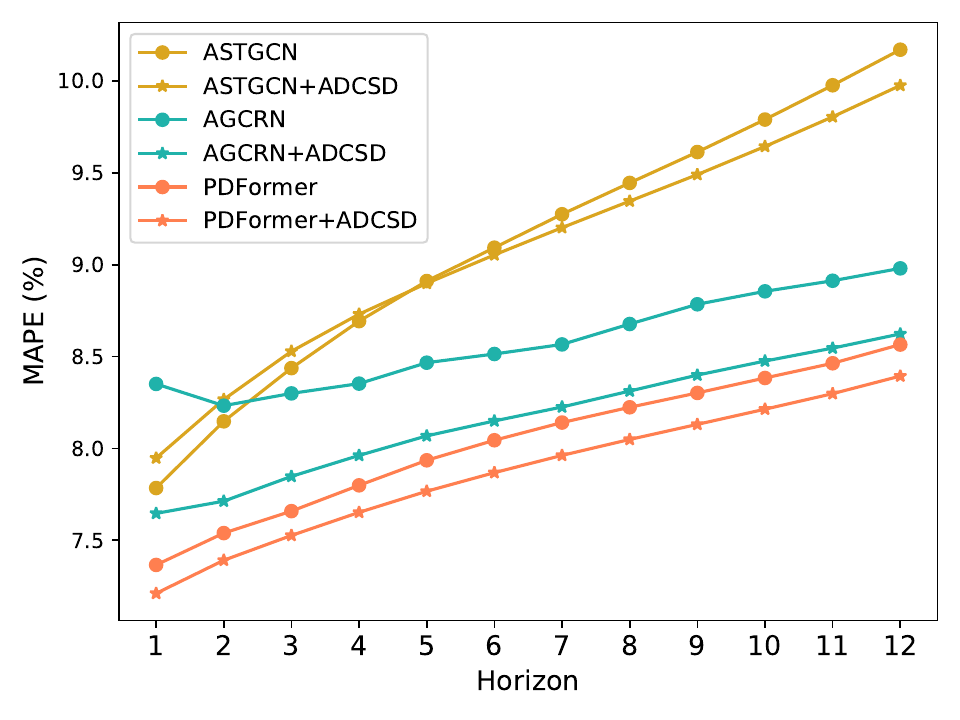}}
    \subfigure[RMSE]{\includegraphics[width=0.32\textwidth]{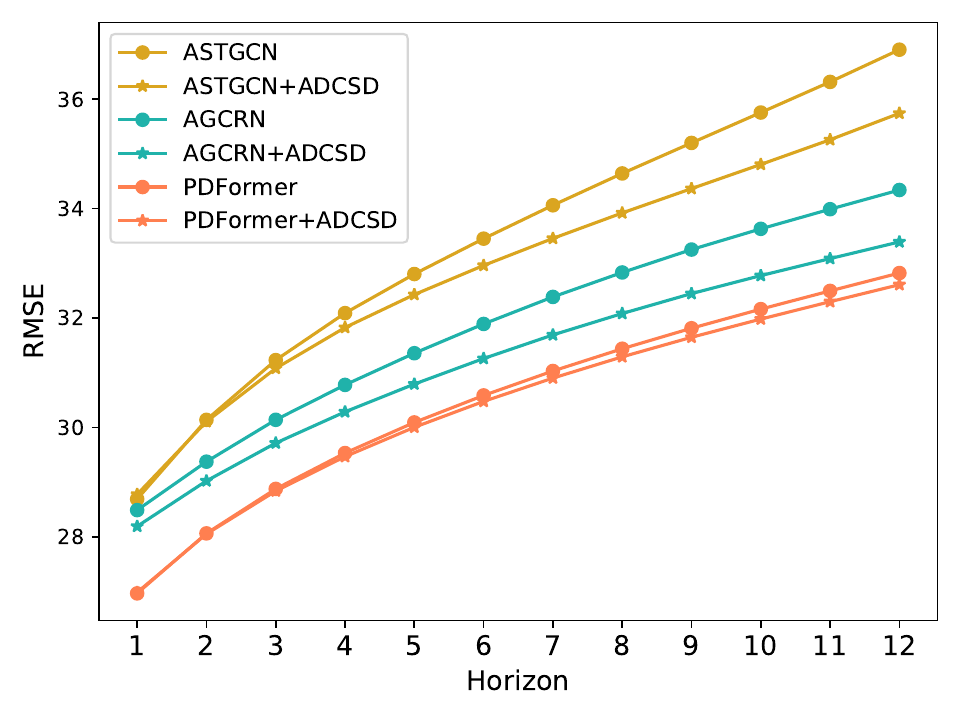}}
  \caption{Prediction performance comparison at each horizon on the PeMS07 dataset.}
\label{fig:each_step}
\end{figure*}

\begin{figure*} [t]
  \centering
    \subfigure[MAE]{\includegraphics[width=0.32\textwidth]{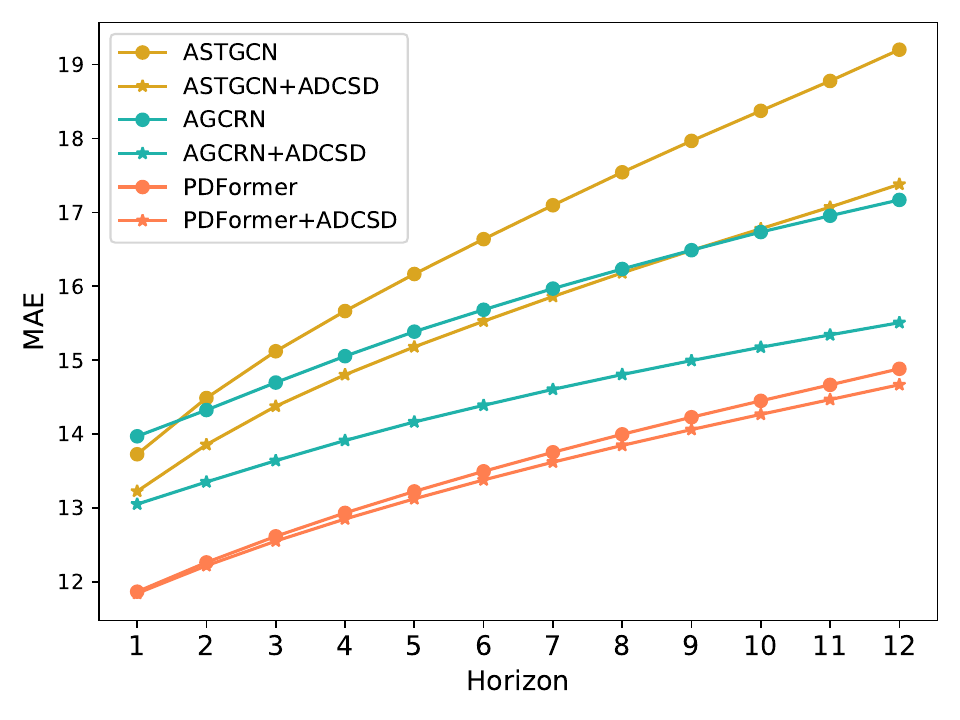}}
    \subfigure[MAPE]{\includegraphics[width=0.32\textwidth]{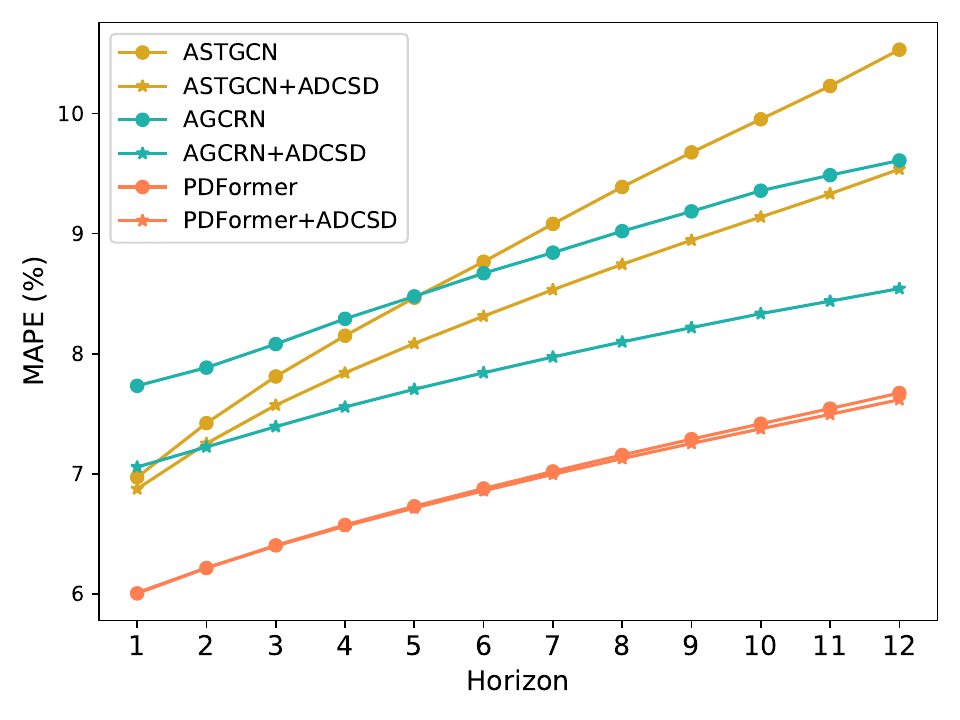}}
    \subfigure[RMSE]{\includegraphics[width=0.32\textwidth]{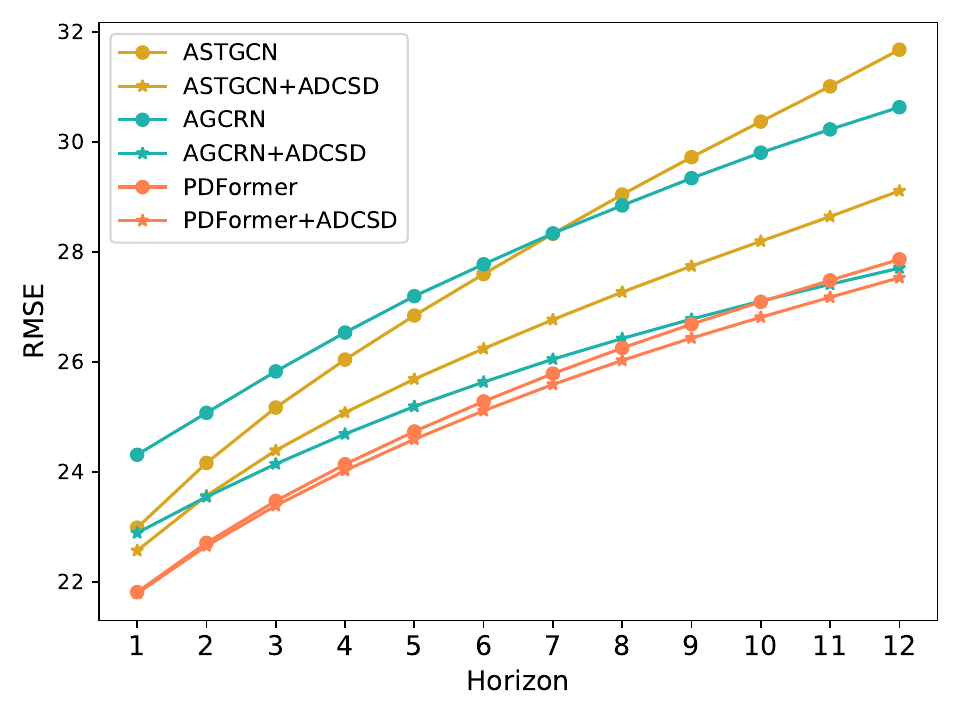}}
  \caption{Prediction performance comparison at each horizon on the BayArea dataset.}
\label{app_fig:each_step}
\end{figure*}

Figure \ref{fig:each_step} further shows the prediction performance at each horizon on PeMS07. With our method attached behind, the performance of the trained models is further improved for almost all horizons, which indicates our method balances short-term and long-term prediction well. Furthermore, our method can achieve more improvement for long-term prediction compared with short-term prediction. This is because, with a longer prediction horizon, the data is more likely to drift away; thus our ADCSD shows its more correction ability from adaptively weighting the seasonal and trend-cyclical
corrections.
Figure \ref{app_fig:each_step} shows the prediction performance at each horizon on the BayArea dataset. Similar to the results on the PeMS07 dataset (see Figure \ref{fig:each_step}), the performance of the trained models improve for almost all horizons with the incorporation of our method, which indicates our method balances short-term and long-term prediction well. In addition, compared to short-term prediction, our method is capable of achieving large improvement in long-term forecasting, which is similar to the phenomenon on the PeMS07 dataset (see Figure \ref{fig:each_step}). Furthermore, our approach can achieve significant performance improvements when the performance of the original model is not so good.

\subsubsection{Grid-based Datasets}
The comparison results with different baselines in the grid-based datasets are shown in Table \ref{tab:main_result_2}. It can be seen from this table that there are the same experimental phenomena as on the graph-based datasets (i.e., Table \ref{tab:main_result_1}), but a slight difference. That is, no matter the performance of the original trained model is good or bad, our method has a stable and relatively obvious improvement, which is different from the phenomenon on the graph-based datasets that there is a higher improvement on models with poor performance and less improvement on models with good performance.

\subsection{Ablation Study}
\label{sec:ablation_study}

To verify the effectiveness of each component in the proposed ADCSD model, we conduct a comprehensive ablation study on both graph-based and grid-based datasets. 

\subsubsection{Graph-based Datasets}
Table \ref{tab:ablation_1} shows ablation study of the graph-based datasets. The model M0 with only the original output $\mathbf{o}$ denotes the results of the model trained on the historical data without any modification. M5 denotes our method. M6 denotes no decomposition on time series. Based on these results, we can conclude the following: (1) \textbf{\textit{Original output matters}}: According to the results between M1 and M2 ($+\mathbf{o}$), there is a significant decrease in performance when the original output is absent. This suggests that the original output plays a crucial role in achieving accurate predictions.
(2) \textbf{\textit{Adaptive vectors matter}}: By comparing M2 and M5 ($+ \bm{\lambda}^s, \bm{\lambda}^t$), we can see that the absence of adaptive vectors leads to a decrease in performance. This finding highlights the importance of adaptive vectors as a critical component of our method. That is because different \pengxin{nodes} have different levels of temporal drift and it's necessary to assign different weights to different \pengxin{nodes}. 
(3) \textbf{\textit{Decomposition matters}}: By comparing M5 and M6, we can see that the absence of decomposing time series leads to a decrease in performance, which indicates the decomposing strategy is helpful. Comparing M3-M5 with M0, we observe that correcting the seasonal or the trend-cyclical part separately can improve the performance, but combining both parts leads to a more significant enhancement in performance, which demonstrates both these two parts are crucial for the performance improvement of our method.

\begin{table}[h]
\caption{Ablation study on the \textbf{graph}-based datasets.}
\resizebox{\linewidth}{!}{
\begin{tabular}{c|c|ccccc|ccc|ccc}
\toprule 
& \multirow{2}{*}{Method} & \multirow{2}{*}{$\mathbf{o}$} & \multirow{2}{*}{$\hat{\mathbf{o}}^s$} & \multirow{2}{*}{$\hat{\mathbf{o}}^t$} & \multirow{2}{*}{$\bm{\lambda}^s$} & \multirow{2}{*}{$\bm{\lambda}^t$} & \multicolumn{3}{|c}{PeMS07} & \multicolumn{3}{|c}{BayArea}  \\
\cmidrule(r){8-13}  
& & & & & & & MAE & MAPE(\%) & RMSE & MAE & MAPE(\%) & RMSE   \\
\midrule
\multirow{7}{*}{\rotatebox{90}{ASTGCN}} & M0 & \checkmark & & & & & 23.708 & 10.170 & 36.903 & 19.201 & 10.531 & 31.678 \\
& M1 & & \checkmark & \checkmark & & & 30.277 & 20.902 & 48.528 & 21.158 & 15.264 & 33.164 \\
& M2 & \checkmark & \checkmark & \checkmark & & & 23.604 & 11.752 & 36.351 & 18.020 & 10.333 & 29.422 \\
& M3 & \checkmark & \checkmark &  & \checkmark &  & 23.329 & 10.193 & 36.213 & 18.220 & 10.432 & 30.022 \\
& M4 & \checkmark &  & \checkmark &  & \checkmark & 22.976 & 9.954 & 35.822 & 17.404 & 9.553 & 29.135 \\
& M5 & \checkmark & \checkmark & \checkmark & \checkmark & \checkmark & 22.910 & 9.975 & 35.740 & 17.380 &  9.536 & 29.111 \\
& M6 & \multicolumn{5}{l|}{$\hat{\mathbf{y}} = \mathbf{o} + \bm{\lambda} \hat{\mathbf{o}}$} & 23.077 & 10.045 & 36.034 & 18.031 & 9.763 & 29.899 \\
\cmidrule(r){1-13}  

\multirow{7}{*}{\rotatebox{90}{AGCRN}} &  M0 &\checkmark & & & & & 20.700 & 8.980 & 34.338 & 17.169 & 9.609 & 30.632 \\
& M1 & & \checkmark & \checkmark & & & 28.495 & 22.191 & 45.410 & 19.076 & 14.359 & 30.887 \\
& M2 & \checkmark & \checkmark & \checkmark & & & 21.358 & 9.952 & 34.764 & 17.279 & 10.509 & 28.530 \\
& M3 & \checkmark & \checkmark &  & \checkmark &  & 20.433 & 8.784 & 33.723 & 16.181 & 9.273 & 28.634 \\
& M4 & \checkmark &  & \checkmark &  & \checkmark & 20.257 & 8.635 & 33.443 & 15.514 & 8.548 & 27.727 \\
& M5 & \checkmark & \checkmark & \checkmark & \checkmark & \checkmark & 20.211 & 8.623 & 33.389 & 15.507 & 8.542 & 27.706 \\
& M6 & \multicolumn{5}{l|}{$\hat{\mathbf{y}} = \mathbf{o} + \bm{\lambda} \hat{\mathbf{o}}$} & 20.404 & 8.681 & 33.777 & 16.033 & 9.025 & 28.334 \\
\cmidrule(r){1-13}  

\multirow{7}{*}{\rotatebox{90}{PDFormer}} &  M0 &\checkmark & & & & & 19.802 & 8.565 & 32.820 & 14.883 & 7.673 & 27.866 \\
& M1 & & \checkmark & \checkmark & & & 24.346 & 17.619 & 37.270 & 17.219 & 10.399 & 29.526 \\
& M2 & \checkmark & \checkmark & \checkmark & & & 20.053 & 9.141 & 32.855 & 15.567 & 8.419 & 27.707 \\
& M3 & \checkmark & \checkmark &  & \checkmark &  & 19.715 & 8.474 & 32.689 & 14.786 & 7.699 & 27.717 \\
& M4 & \checkmark &  & \checkmark &  & \checkmark & 19.654 & 8.447 & 32.634 & 14.670 & 7.618 & 27.535 \\
& M5 & \checkmark & \checkmark & \checkmark & \checkmark & \checkmark & 19.627 & 8.393 & 32.607 & 14.667 & 7.617 & 27.531 \\
& M6 & \multicolumn{5}{l|}{$\hat{\mathbf{y}} = \mathbf{o} + \bm{\lambda} \hat{\mathbf{o}}$} & 19.717 & 8.448 & 32.688 & 14.677 & 7.620 & 27.633 \\

\bottomrule
\end{tabular}
}
\label{tab:ablation_1}
\end{table}

\begin{table*}[t]
\centering
\caption{Ablation study on the \textbf{grid}-based datasets.}
\resizebox{\linewidth}{!}{
\begin{tabular}{c|c|ccccc|cccccc|cccccc}
\toprule 
& & & & & & & \multicolumn{6}{|c}{NYCTaxi} & \multicolumn{6}{|c}{T-Drive}  \\
\cmidrule(r){8-19}  
& Method & {$\mathbf{o}$} & {$\hat{\mathbf{o}}^s$} & {$\hat{\mathbf{o}}^t$} & {$\bm{\lambda}^s$} & {$\bm{\lambda}^t$} & \multicolumn{3}{|c}{inflow} & \multicolumn{3}{|c}{outflow} & \multicolumn{3}{|c}{inflow} & \multicolumn{3}{|c}{outflow}  \\
\cmidrule(r){8-19}  
& & & & & & & MAE & MAPE(\%) & RMSE & MAE & MAPE(\%) & RMSE & MAE & MAPE(\%) & RMSE & MAE & MAPE(\%) & RMSE  \\
\midrule
\multirow{7}{*}{\rotatebox{90}{ASTGCN}} & M0 & \checkmark & & & & & 25.302 & 23.405 & 43.228 & 21.720 & 22.515 & 37.758 & 41.578 & 32.410 & 76.801 & 41.376 & 32.385 & 76.365 \\
& M1 & & \checkmark & \checkmark & & & 31.368 & 31.479 & 57.114 & 28.392 & 31.418 & 51.594 & 72.957 & 65.586 & 140.727 & 68.760 & 45.872 & 130.298 \\
& M2 & \checkmark & \checkmark & \checkmark & & & 25.401 & 24.403 & 42.942 & 22.802 & 26.713 & 37.940 & 42.843 & 37.484 & 77.629 & 47.784 & 49.342 & 81.590 \\
& M3 & \checkmark & \checkmark &  & \checkmark &  & 25.204 & 23.339 & 43.164 & 21.563 & 22.521 & 37.611 & 41.529 & 32.395 & 76.740 & 41.313 & 32.388 & 76.283 \\
& M4 & \checkmark &  & \checkmark &  & \checkmark & 25.133 & 23.252 & 42.991 & 21.547 & 22.523 & 37.446 & 41.523 & 32.348 & 76.694 & 41.294 & 32.322 & 76.207 \\
& M5 & \checkmark & \checkmark & \checkmark & \checkmark & \checkmark & {25.117} & {23.280} & {42.941} & {21.480} & {22.487} & {37.315} & {41.499} & {32.328} & {76.675} & {41.248} & {32.299} & {76.166}  \\
& M6 & \multicolumn{5}{l|}{$\hat{\mathbf{y}} = \mathbf{o} + \bm{\lambda} \hat{\mathbf{o}}$} & 25.131 & 23.316 & 42.941 & 21.559 & 22.586 & 37.428 & 41.503 & 32.388 & 76.786 & 41.332 & 32.315 & 76.223 \\
\cmidrule(r){1-19} 

\multirow{7}{*}{\rotatebox{90}{AGCRN}} & M0 & \checkmark & & & & &  18.299 & 16.445 & 31.948 & 15.754 & 16.211 & 27.383 & 21.374 &  16.478 & 41.066 & 21.395 & 16.459 & 41.081 \\
& M1 & & \checkmark & \checkmark & & & 25.553 & 25.020 & 48.699 & 21.324 & 25.699 & 37.104 & 59.871 & 53.029 & 109.597 & 61.660 & 36.773 & 136.090 \\
& M2 & \checkmark & \checkmark & \checkmark & & & 18.611 & 17.605 & 31.971 & 16.551 & 18.166 & 28.632 & 33.591 & 49.569 & 52.663 & 29.340 & 36.276 & 52.176 \\
& M3 & \checkmark & \checkmark &  & \checkmark &  & 18.195 & 16.289 & 31.807 & 15.649 & 16.077 & 27.210 & 21.285 & 16.411 & 40.983 & 21.307 & 16.393 & 41.003 \\
& M4 & \checkmark &  & \checkmark &  & \checkmark & 18.125 & 16.270 & 31.679 & 15.591 & 16.018 & 27.138 & 21.263 & 16.411 & 40.879 & 21.306 & 16.397 & 40.940 \\
& M5 & \checkmark & \checkmark & \checkmark & \checkmark & \checkmark & {18.114} & {16.232} & {31.673} & {15.569} & {16.001} & {27.102} & {21.242} & {16.389} & {40.856} & {21.271} & {16.366} & {40.897} \\
& M6 & \multicolumn{5}{l|}{$\hat{\mathbf{y}} = \mathbf{o} + \bm{\lambda} \hat{\mathbf{o}}$} & 18.114 & 16.326 & 31.691 & 15.621 & 16.129 & 27.112 & 21.304 & 16.447 & 41.053 & 21.338 & 16.406 & 41.092 \\
\cmidrule(r){1-19} 

\multirow{7}{*}{\rotatebox{90}{PDFormer}} & M0 & \checkmark & & & & & 17.150 & 15.320 & 30.185 & 14.778 & 15.055 & 25.903 & 21.481 & 17.504 & 38.889 & 21.515 & 17.507 & 38.914 \\
& M1 & & \checkmark & \checkmark & & & 30.361 & 28.484 & 66.836 & 24.893 & 27.800 & 55.012 & 76.860 & 51.039 & 174.795 & 74.451 & 47.152 & 164.612 \\
& M2 & \checkmark & \checkmark & \checkmark & & & 18.901 & 18.371 & 32.556 & 16.835 & 19.531 & 28.490 & 26.191 & 21.729 & 46.951 & 31.702 & 29.496 & 54.644 \\
& M3 & \checkmark & \checkmark &  & \checkmark &  & 17.069 & 15.274 & 30.034 & 14.731 & 15.035 & 25.727 & 21.395 & 17.411 & 38.821 & 21.417 & 17.403 & 38.832 \\
& M4 & \checkmark &  & \checkmark &  & \checkmark & 16.999 & 15.189 & 29.941 & 14.608 & 14.919 & 25.487 & 21.343 & 17.327 & 38.674 & 21.379 & 17.324 & 38.709 \\
& M5 & \checkmark & \checkmark & \checkmark & \checkmark & \checkmark & {16.987} & {15.181} & {29.928} & {14.595} & {14.912} & {25.460} & {21.308} & {17.283} & {38.664} & {21.337} & {17.274} &  {38.694} \\
& M6 & \multicolumn{5}{l|}{$\hat{\mathbf{y}} = \mathbf{o} + \bm{\lambda} \hat{\mathbf{o}}$} & 17.001 & 15.201 & 30.020 & 14.673 & 14.976 & 25.653 & 21.412 & 17.529 & 38.715 &  21.412 & 17.493 & 38.701 \\

\bottomrule
\end{tabular}
}
\label{app_tab:ablation_2}
\end{table*}

\subsubsection{Grid-based Datasets}
\label{appx: grid-ablation}

Experimental results of ablation study on the grid-based datasets are shown in Table \ref{app_tab:ablation_2}. The phenomena are similar to the results on the graph-based dataset (see Table \ref{tab:ablation_1}). That is, (1) \textbf{\textit{Original output matters}}: The original output plays a crucial role in achieving accurate predictions. (2) \textbf{\textit{Adaptive vectors matter}}: The adaptive vectors are a critical component of our method. (3) \textbf{\textit{Decomposition matters}}: The decomposing strategy is helpful. Both the seasonal and trend-cyclical parts are crucial to the performance improvement of our method. 

\begin{table}[t]
\caption{Computational cost on the NYCTaxi dataset.}
\resizebox{\linewidth}{!}{
\begin{tabular}{l|cc|cc|cc}
\toprule 
\multirow{2}{*}{Method}  & \multicolumn{2}{|c}{ASTGCN} & \multicolumn{2}{|c}{AGCRN} & \multicolumn{2}{|c}{PDFormer}  \\
\cmidrule(r){2-7}  
& \# Parameters & \# Time  & \# Parameters & \# Time & \# Parameters & \# Time  \\
\midrule
Test & - & 43s & - & 35s & - & 63s \\
TENT \cite{wang2021tent} & 72542 & 133s & 750330 & 127s & 531628 & 307s \\
TTT-MAE \cite{gandelsman2022test}  & 72542 & 144s & 750330 & 128s & 531628 & 344s \\
ADCSD & 231834 & 97s & 231834 & 76s & 231834 & 132s \\ 

\bottomrule
\end{tabular}
}
\label{tab:compution_cost}
\end{table}

\subsection{Computational Cost}

To evaluate the computational cost, we compare the parameter numbers and testing time of our method with TENT and TTT-MAE on the NYCTaxi dataset in Table \ref{tab:compution_cost}. In terms of the training time, our method runs faster than all compared baseline methods and only slightly slower than the direct test method, which demonstrates the efficiency of our method. Furthermore, the training parameters of our method are less than OTTA baselines when the trained model is AGCRN or PDFormer. That is because the training parameters of baselines  are related to the trained model. The more parameters of the trained model, the more parameters baselines are required, while the parameters of our method are fixed and kept at a small value.

\section{Conclusion}
\label{sec:conclu}

\pengxin{In this paper}, to tackle the temporal drift problem in spatial-temporal \pengxin{traffic flow} forecasting, we conduct the first study to apply OTTA to the spatial-temporal \pengxin{traffic flow} forecasting problem. In order to adapt to the characteristics of \pengxin{traffic flow} data, we propose the ADCSD model, \pengxin{which first decomposes the output of a trained model into seasonal and trend-cyclical parts and then corrects them by two separate modules during the testing phase using the latest observed data}. Based on this operation, the model trained on the historical data can better adapt to future data. 
The proposed method is a lite network that can be universally attached with main-stream \pengxin{traffic flow forecasting} deep models as a plug-and-play component, and it could significantly improve their performance. 
Applying OTTA to spatial-temporal \pengxin{traffic flow} forecasting to solve the temporal drift problem is practical since it fits the characteristics of time series data, and we hope more researchers can pay attention to this practical setting. 
In our future study, we are interested in applying the proposed ADCSD model to more applications in intelligent transportation.

\section*{Acknowledgments}

This work is partially supported by the National Key R\&D Program of China (NO.2022ZD0160101) and Shenzhen fundamental research program JCYJ20210324105000003.


 
%

\bibliography{main}

\begin{thebibliography}{10}
\providecommand{\url}[1]{#1}
\csname url@samestyle\endcsname
\providecommand{\newblock}{\relax}
\providecommand{\bibinfo}[2]{#2}
\providecommand{\BIBentrySTDinterwordspacing}{\spaceskip=0pt\relax}
\providecommand{\BIBentryALTinterwordstretchfactor}{4}
\providecommand{\BIBentryALTinterwordspacing}{\spaceskip=\fontdimen2\font plus
\BIBentryALTinterwordstretchfactor\fontdimen3\font minus \fontdimen4\font\relax}
\providecommand{\BIBforeignlanguage}[2]{{%
\expandafter\ifx\csname l@#1\endcsname\relax
\typeout{** WARNING: IEEEtran.bst: No hyphenation pattern has been}%
\typeout{** loaded for the language `#1'. Using the pattern for}%
\typeout{** the default language instead.}%
\else
\language=\csname l@#1\endcsname
\fi
#2}}
\providecommand{\BIBdecl}{\relax}
\BIBdecl

\bibitem{10323235}
X.~Liu, X.~Qin, M.~Zhou, H.~Sun, and S.~Han, ``Community-based dandelion algorithm-enabled feature selection and broad learning system for traffic flow prediction,'' \emph{IEEE Trans. Intell. Transp. Syst.}, pp. 1--14, 2023.

\bibitem{lv2014traffic}
Y.~Lv, Y.~Duan, W.~Kang, Z.~Li, and F.-Y. Wang, ``Traffic flow prediction with big data: A deep learning approach,'' \emph{IEEE Trans. Intell. Transp. Syst.}, vol.~16, no.~2, pp. 865--873, 2014.

\bibitem{zhao2019t}
L.~Zhao, Y.~Song, C.~Zhang, Y.~Liu, P.~Wang, T.~Lin, M.~Deng, and H.~Li, ``T-gcn: A temporal graph convolutional network for traffic prediction,'' \emph{IEEE Trans. Intell. Transp. Syst.}, vol.~21, no.~9, pp. 3848--3858, 2019.

\bibitem{tedjopurnomo2020survey}
D.~A. Tedjopurnomo, Z.~Bao, B.~Zheng, F.~M. Choudhury, and A.~K. Qin, ``A survey on modern deep neural network for traffic prediction: Trends, methods and challenges,'' \emph{IEEE Trans. Knowl. Data Eng.}, vol.~34, no.~4, pp. 1544--1561, 2020.

\bibitem{jiang2022graph}
W.~Jiang and J.~Luo, ``Graph neural network for traffic forecasting: A survey,'' \emph{Expert Syst. Appl.}, vol. 207, p. 117921, 2022.

\bibitem{nagy2018survey}
A.~M. Nagy and V.~Simon, ``Survey on traffic prediction in smart cities,'' \emph{Pervasive Mob. Comput.}, vol.~50, pp. 148--163, 2018.

\bibitem{guo2019attention}
S.~Guo, Y.~Lin, N.~Feng, C.~Song, and H.~Wan, ``Attention based spatial-temporal graph convolutional networks for traffic flow forecasting,'' in \emph{Proc. AAAI Conf. Artif. Intell.}, vol.~33, no.~01, 2019, pp. 922--929.

\bibitem{bai2020adaptive}
L.~Bai, L.~Yao, C.~Li, X.~Wang, and C.~Wang, ``Adaptive graph convolutional recurrent network for traffic forecasting,'' \emph{Proc. Adv. Neural Inf. Process. Syst. (NeurIPS)}, vol.~33, pp. 17\,804--17\,815, 2020.

\bibitem{woo2022cost}
G.~Woo, C.~Liu, D.~Sahoo, A.~Kumar, and S.~Hoi, ``Cost: Contrastive learning of disentangled seasonal-trend representations for time series forecasting,'' in \emph{Proc. Int. Conf. Learn. Represent. (ICLR)}, 2022.

\bibitem{pdformer}
J.~Jiang, C.~Han, W.~X. Zhao, and J.~Wang, ``Pdformer: Propagation delay-aware dynamic long-range transformer for traffic flow prediction,'' in \emph{Proc. AAAI Conf. Artif. Intell.}, 2023.

\bibitem{kuznetsov2014generalization}
V.~Kuznetsov and M.~Mohri, ``Generalization bounds for time series prediction with non-stationary processes,'' in \emph{Proc. Algorithmic Learning Theory (ALT)}, 2014, pp. 260--274.

\bibitem{du2021adarnn}
Y.~Du, J.~Wang, W.~Feng, S.~Pan, T.~Qin, R.~Xu, and C.~Wang, ``Adarnn: Adaptive learning and forecasting of time series,'' in \emph{Proc. Conf. Inf. Know. Mana. (CIKM)}, 2021, pp. 402--411.

\bibitem{duan2023combating}
W.~Duan, X.~He, L.~Zhou, L.~Thiele, and H.~Rao, ``Combating distribution shift for accurate time series forecasting via hypernetworks,'' in \emph{Proc. Int. Conf. Para. Dist. Syst. (ICPADS)}.\hskip 1em plus 0.5em minus 0.4em\relax IEEE, 2023, pp. 900--907.

\bibitem{liang2023comprehensive}
J.~Liang, R.~He, and T.~Tan, ``A comprehensive survey on test-time adaptation under distribution shifts,'' Preprint arXiv:2303.15361, 2023.

\bibitem{sun2020test}
Y.~Sun, X.~Wang, Z.~Liu, J.~Miller, A.~Efros, and M.~Hardt, ``Test-time training with self-supervision for generalization under distribution shifts,'' in \emph{Proc. Int. Conf. Mach. Learn. (ICML)}.\hskip 1em plus 0.5em minus 0.4em\relax PMLR, 2020, pp. 9229--9248.

\bibitem{niu2022efficient}
S.~Niu, J.~Wu, Y.~Zhang, Y.~Chen, S.~Zheng, P.~Zhao, and M.~Tan, ``Efficient test-time model adaptation without forgetting,'' in \emph{Proc. Int. Conf. Mach. Learn. (ICML)}.\hskip 1em plus 0.5em minus 0.4em\relax PMLR, 2022, pp. 16\,888--16\,905.

\bibitem{wang2021tent}
D.~Wang, E.~Shelhamer, S.~Liu, B.~Olshausen, and T.~Darrell, ``Tent: Fully test-time adaptation by entropy minimization,'' in \emph{Proc. Int. Conf. Learn. Represent. (ICLR)}, 2021.

\bibitem{gandelsman2022test}
Y.~Gandelsman, Y.~Sun, X.~Chen, and A.~Efros, ``Test-time training with masked autoencoders,'' \emph{Proc. Adv. Neural Inf. Process. Syst. (NeurIPS)}, vol.~35, pp. 29\,374--29\,385, 2022.

\bibitem{gong2022note}
T.~Gong, J.~Jeong, T.~Kim, Y.~Kim, J.~Shin, and S.-J. Lee, ``Note: Robust continual test-time adaptation against temporal correlation,'' \emph{Proc. Adv. Neural Inf. Process. Syst. (NeurIPS)}, vol.~35, pp. 27\,253--27\,266, 2022.

\bibitem{wang2022continual}
Q.~Wang, O.~Fink, L.~Van~Gool, and D.~Dai, ``Continual test-time domain adaptation,'' in \emph{Proc. IEEE/CVF Conf. Comput. Vis. Pattern Recognit. (CVPR)}, 2022, pp. 7201--7211.

\bibitem{tian2021spatial}
C.~Tian and W.~K. Chan, ``Spatial-temporal attention wavenet: A deep learning framework for traffic prediction considering spatial-temporal dependencies,'' \emph{IET Intell. Transp. Syst.}, vol.~15, no.~4, pp. 549--561, 2021.

\bibitem{clark2022meta}
K.~Clark, K.~Guu, M.-W. Chang, P.~Pasupat, G.~Hinton, and M.~Norouzi, ``Meta-learning fast weight language models,'' Preprint arXiv:2212.02475, 2022.

\bibitem{cleveland1990stl}
R.~B. Cleveland, W.~S. Cleveland, J.~E. McRae, and I.~Terpenning, ``Stl: A seasonal-trend decomposition,'' \emph{J. Off. Stat.}, vol.~6, no.~1, pp. 3--73, 1990.

\bibitem{hyndman2018forecasting}
R.~J. Hyndman and G.~Athanasopoulos, \emph{Forecasting: principles and practice}.\hskip 1em plus 0.5em minus 0.4em\relax OTexts, 2018.

\bibitem{li2020long}
Z.~Li, H.~Yan, C.~Zhang, and F.~Tsung, ``Long-short term spatiotemporal tensor prediction for passenger flow profile,'' \emph{IEEE Robot. Autom. Lett.}, vol.~5, no.~4, pp. 5010--5017, 2020.

\bibitem{amato2020novel}
F.~Amato, F.~Guignard, S.~Robert, and M.~Kanevski, ``A novel framework for spatio-temporal prediction of environmental data using deep learning,'' \emph{Sci Rep}, vol.~10, no.~1, p. 22243, 2020.

\bibitem{liu2022msdr}
D.~Liu, J.~Wang, S.~Shang, and P.~Han, ``Msdr: Multi-step dependency relation networks for spatial temporal forecasting,'' in \emph{Proc. SIGKDD Conf. Know. Disc. \& Data Mining}, 2022, pp. 1042--1050.

\bibitem{jones2017machine}
N.~Jones, ``How machine learning could help to improve climate forecasts,'' \emph{Nature}, vol. 548, no. 7668, 2017.

\bibitem{longo2017crowd}
A.~Longo, M.~Zappatore, M.~Bochicchio, and S.~B. Navathe, ``Crowd-sourced data collection for urban monitoring via mobile sensors,'' \emph{ACM Trans. Internet. Technol.}, vol.~18, no.~1, pp. 1--21, 2017.

\bibitem{guo2021learning}
S.~Guo, Y.~Lin, H.~Wan, X.~Li, and G.~Cong, ``Learning dynamics and heterogeneity of spatial-temporal graph data for traffic forecasting,'' \emph{IEEE Trans. Knowl. Data Eng.}, vol.~34, no.~11, pp. 5415--5428, 2021.

\bibitem{you2021learning}
X.~You, M.~Zhang, D.~Ding, F.~Feng, and Y.~Huang, ``Learning to learn the future: Modeling concept drifts in time series prediction,'' in \emph{Proc. Conf. Inf. Know. Mana. (CIKM)}, 2021, pp. 2434--2443.

\bibitem{arik2022self}
S.~O. Arik, N.~C. Yoder, and T.~Pfister, ``Self-adaptive forecasting for improved deep learning on non-stationary time-series,'' Preprint arXiv:2202.02403, 2022.

\bibitem{kim2022reversible}
T.~Kim, J.~Kim, Y.~Tae, C.~Park, J.-H. Choi, and J.~Choo, ``Reversible instance normalization for accurate time-series forecasting against distribution shift,'' in \emph{Proc. Int. Conf. Learn. Represent. (ICLR)}, 2022.

\bibitem{wang2023koopman}
R.~Wang, Y.~Dong, S.~O. Arik, and R.~Yu, ``Koopman neural operator forecaster for time-series with temporal distributional shifts,'' in \emph{Proc. Int. Conf. Learn. Represent. (ICLR)}, 2023.

\bibitem{bai2023temporal}
G.~Bai, C.~Ling, and L.~Zhao, ``Temporal domain generalization with drift-aware dynamic neural networks,'' in \emph{Proc. Int. Conf. Learn. Represent. (ICLR)}, 2023.

\bibitem{hu2021mixnorm}
X.~Hu, G.~Uzunbas, S.~Chen, R.~Wang, A.~Shah, R.~Nevatia, and S.-N. Lim, ``Mixnorm: Test-time adaptation through online normalization estimation,'' Preprint arXiv:2110.11478, 2021.

\bibitem{azimi2022self}
F.~Azimi, S.~Palacio, F.~Raue, J.~Hees, L.~Bertinetto, and A.~Dengel, ``Self-supervised test-time adaptation on video data,'' in \emph{Proc. IEEE/CVF Int. Conf. Comput. Vis. (ICCV)}, 2022, pp. 3439--3448.

\bibitem{hong2023mecta}
J.~Hong, L.~Lyu, J.~Zhou, and M.~Spranger, ``Mecta: Memory-economic continual test-time model adaptation,'' in \emph{Proc. Int. Conf. Learn. Represent. (ICLR)}, 2023.

\bibitem{boudiaf2022parameter}
M.~Boudiaf, R.~Mueller, I.~Ben~Ayed, and L.~Bertinetto, ``Parameter-free online test-time adaptation,'' in \emph{Proc. IEEE/CVF Conf. Comput. Vis. Pattern Recognit. (CVPR)}, 2022, pp. 8344--8353.

\bibitem{ba2016layer}
J.~L. Ba, J.~R. Kiros, and G.~E. Hinton, ``Layer normalization,'' Preprint arXiv:1607.06450, 2016.

\bibitem{hendrycks2016gaussian}
D.~Hendrycks and K.~Gimpel, ``Gaussian error linear units (gelus),'' Preprint arXiv:1606.08415, 2016.

\bibitem{song2020spatial}
C.~Song, Y.~Lin, S.~Guo, and H.~Wan, ``Spatial-temporal synchronous graph convolutional networks: A new framework for spatial-temporal network data forecasting,'' in \emph{Proc. AAAI Conf. Artif. Intell.}, vol.~34, no.~01, 2020, pp. 914--921.

\bibitem{liu2020dynamic}
L.~Liu, J.~Zhen, G.~Li, G.~Zhan, Z.~He, B.~Du, and L.~Lin, ``Dynamic spatial-temporal representation learning for traffic flow prediction,'' \emph{IEEE Trans. Intell. Transp. Syst.}, vol.~22, no.~11, pp. 7169--7183, 2020.

\bibitem{pan2019urban}
Z.~Pan, Y.~Liang, W.~Wang, Y.~Yu, Y.~Zheng, and J.~Zhang, ``Urban traffic prediction from spatio-temporal data using deep meta learning,'' in \emph{Proc. SIGKDD Conf. Know. Disc. \& Data Mining}, 2019, pp. 1720--1730.

\bibitem{chen2001freeway}
C.~Chen, K.~Petty, A.~Skabardonis, P.~Varaiya, and Z.~Jia, ``Freeway performance measurement system: mining loop detector data,'' \emph{Transp. Res. Record}, vol. 1748, no.~1, pp. 96--102, 2001.

\bibitem{wang2021libcity}
J.~Wang, J.~Jiang, W.~Jiang, C.~Li, and W.~X. Zhao, ``Libcity: An open library for traffic prediction,'' in \emph{Proc. ACM SIGSPATIAL Conf.}, 2021, pp. 145--148.

\bibitem{kingma2014adam}
D.~P. Kingma and J.~Ba, ``Adam: A method for stochastic optimization,'' Preprint arXiv:1412.6980, 2014.

\end{thebibliography}
\bibliographystyle{IEEEtran}

\vfill

\end{document}